\documentclass[english]{article}
\usepackage[latin9]{inputenc}
\usepackage{array}
\usepackage{url}
\usepackage{multirow}
\usepackage{amsmath}
\usepackage{amsthm}
\usepackage{amssymb}
\usepackage{graphicx}

\makeatletter

\providecommand{\tabularnewline}{\\}

\theoremstyle{plain}
\newtheorem{thm}{\protect\theoremname}
\theoremstyle{remark}
\newtheorem{rem}[thm]{\protect\remarkname}
\theoremstyle{plain}
\newtheorem{prop}[thm]{\protect\propositionname}
\theoremstyle{definition}
\newtheorem{defn}[thm]{\protect\definitionname}

\@ifundefined{date}{}{\date{}}
\usepackage[preprint]{neurips_2022}

\usepackage{tcolorbox}

\usepackage{url}
\usepackage{hyperref}
\usepackage[linesnumbered, ruled, vlined]{algorithm2e}
\renewcommand{\cite}{\citep}

\makeatother

\usepackage{babel}
\providecommand{\definitionname}{Definition}
\providecommand{\propositionname}{Proposition}
\providecommand{\remarkname}{Remark}
\providecommand{\theoremname}{Theorem}

\begin{document}
\title{Stable Hadamard Memory: Revitalizing Memory-Augmented Agents for Reinforcement
Learning}
\author{Hung Le{\small{}, }Kien Do, Dung Nguyen, Sunil Gupta, and Svetha Venkatesh{\small{}}\\
{\small{}Applied AI Institute, Deakin University, Geelong, Australia}\\
\texttt{\footnotesize{}thai.le@deakin.edu.au}}
\maketitle
\begin{abstract}
Effective decision-making in partially observable environments demands
robust memory management. Despite their success in supervised learning,
current deep-learning memory models struggle in reinforcement learning
environments that are partially observable and long-term. They fail
to efficiently capture relevant past information, adapt flexibly to
changing observations, and maintain stable updates over long episodes.
We theoretically analyze the limitations of existing memory models
within a unified framework and introduce the Stable Hadamard Memory,
a novel memory model for reinforcement learning agents. Our model
dynamically adjusts memory by erasing no longer needed experiences
and reinforcing crucial ones computationally efficiently. To this
end, we leverage the Hadamard product for calibrating and updating
memory, specifically designed to enhance memory capacity while mitigating
numerical and learning challenges. Our approach significantly outperforms
state-of-the-art memory-based methods on challenging partially observable
benchmarks, such as meta-reinforcement learning, long-horizon credit
assignment, and POPGym, demonstrating superior performance in handling
long-term and evolving contexts.
\end{abstract}

\section{Introduction}

Reinforcement learning agents necessitate memory. This is especially
true in Partially Observable Markov Decision Processes (POMDPs \cite{kaelbling1998planning}),
where past information is crucial for making informed decisions. However,
designing a robust memory remains an enduring challenge, as agents
must not only store long-term memories but also dynamically update
them in response to evolving environments. Memory-augmented neural
networks (MANNs)--particularly those developed for supervised learning
\cite{graves2016hybrid,Vaswani2017AttentionIA}, while offering promising
solutions, have consistently struggled in these dynamic settings.
Recent empirical studies \cite{morad2023popgym,ni2024transformers}
have shown that MANNs exhibit instability and underperform simpler
vector-based memory models such as GRU \cite{chung2014empirical}
or LSTM \cite{hochreiter1997long}. The issue is exacerbated in complex
and sparse reward scenarios where agents must selectively retain and
erase memories based on relevance. Unfortunately, existing methods
fail to provide a memory writing mechanism that is simultaneously
efficient, stable, and flexible to meet these demands.

In this paper, we focus on designing a better writing mechanism to
encode new information into the memory\emph{.} To this end, we introduce
the Hadamard Memory Framework (HMF), a unified model that encompasses
many existing writing methods as specific cases. This framework highlights
the critical role of \emph{memory calibration}, which involves linearly
adjusting memory elements by multiplying the memory matrix with a
\emph{calibration matrix} and then adding an\emph{ update matrix}.
By leveraging Hadamard products that operate element-wise on memory
matrices, we allow memory writing without mixing the memory cells
in a computationally efficient manner. More importantly, the calibration
and update matrices are dynamically computed based on the input at
each step. This enables the model to learn adaptive memory rules,
which are crucial for generalization. For instance, in meta-reinforcement
learning with varying maze layouts, a fixed memory update rule may
work for one layout but fail in another. By allowing the calibration
matrix to adjust according to the current maze observation, the agent
can learn to adapt to any layout configuration. A dynamic calibration
matrix also enables the agent to flexibly forget and later recall
information as needed. For example, an agent navigating a room may
need to remember the key's location, retain it during a detour, and
later recall it when reaching a door, while discarding irrelevant
detour events.

Although most current memory models can be reformulated within the
HMF, they tend to be either overly simplistic with limited calibration
capabilities \cite{katharopoulos2020transformers,radford2019language}
or unable to manage memory writing reliably, suffering from gradient
vanishing or exploding issues \cite{ba2016using,morad2024reinforcement}.
To address these limitations, we propose a specific instance of HMF,
called Stable Hadamard Memory, which introduces a novel calibration
mechanism based on two key principles: (i) dynamically adjusting memory
values in response to the current context input to selectively weaken
outdated or enhance relevant information, and (ii) ensuring the expected
value of the calibration matrix product remains bounded, thereby preventing
gradient vanishing or exploding. Through extensive experimentation
on POMDP benchmarks, including meta reinforcement learning, long-horizon
credit assignment, and hard memorization games, we demonstrate that
our method consistently outperforms state-of-the-art memory-based
models in terms of performance while also delivering competitive speed.
We also provide comprehensive ablation studies that offer insights
into the components and internal workings of our memory models.

\section{Background}

\subsection{Reinforcement Learning Preliminaries}

A Partially Observable Markov Decision Process (POMDP) is formally
defined as a tuple $\langle S,A,O,R,P,\gamma\rangle$, where $S$
is the set of states, $A$ is the set of actions, $O$ is the observation
space, $R:S\times A\to\mathbb{R}$ is the reward function, $P:S\times A\to\Delta(S)$
defines the state transition probabilities, and $\gamma\in[0,1)$
is the discount factor. Here, the agent does not directly observe
the true environment state $s_{t}$. Instead, it receives an observation
$o_{t}\sim O(s_{t})$, which provides partial information about the
state, often not enough for optimal decision making. Therefore, the
agent must make decisions based on its current observation $o_{t}$
and a history of previous observations, actions, and rewards $(a_{0},r_{0},o_{1},\dots,a_{t-1},r_{t-1},o_{t})$.
The history may exclude past rewards or actions. 

Let us denote the input context at timestep $t$ as $x_{t}=\left(o_{t},a_{t-1},r_{t-1}\right)$
and assume that we can encode the sequence of contexts into a memory
$M_{t}=f\left(\left\{ x_{i}\right\} _{i=1}^{t}\right)$. The goal
is to learn a policy $\pi(a_{t}|M_{t})$ that maximizes the expected
cumulative discounted reward:

\begin{equation}
J(\pi)=\mathbb{E}_{\pi}\left[\sum_{t=1}^{\infty}\gamma^{t}R(s_{t},a_{t})|a_{t}\sim\pi(a_{t}|M_{t}),s_{t+1}\sim P(s_{t+1}|s_{t},a_{t}),o_{t}\sim O(s_{t})\right]\label{eq:rlobj}
\end{equation}
Thus, a memory system capable of capturing past experiences is essential
for agents to handle the partial observability of the environment
while maximizing long-term rewards. 

\subsection{Memory-Augmented Neural Networks}

We focus on matrix memory $M$, and to simplify notation, we assume
it is a square matrix. Given a memory $M\in\mathbb{R}^{H\times H}$,
we usually read from the memory as:

\begin{equation}
h_{t}=M_{t}q\left(x_{t}\right)\label{eq:hread}
\end{equation}
where $q$ is a query network $q:\mathbb{R}^{D}\mapsto\mathbb{R}^{H}$
to map an input context $x_{t}\in\mathbb{R}^{D}$ to a query vector.
The read value $h_{t}$ later will be used as the input for policy/value
functions. Even more important than the reading process is the memory
writing:\emph{ How can information be written into the memory to ensure
efficient and accurate memory reading?} A general formulation for
memory writing is $M_{t}=f\left(M_{t-1},x_{t}\right)$ with $f$ as
the update function that characterizes the memory models.

The simplest form of memory writing traces back to Hebbian learning
rules: $M_{t}=M_{t-1}+x_{t}\otimes x_{t}$ where $\otimes$ is outer
product \cite{kohonen1973representation,hebb2005organization}. Later,
researchers have proposed ``fast weight'' memory \cite{marr1991theory,schmidhuber1992learning,ba2016using}:
\begin{equation}
M_{t}=M_{t-1}\lambda+\eta g\left(M_{t-1},x_{t}\right)\otimes g\left(M_{t-1},x_{t}\right)\label{eq:fw}
\end{equation}
where $g$ is a non-linear function that take the previous memory
and the current input data as the input; $\lambda$ and $\eta$ are
constant hyperparameters. On the other hand, computer-inspired memory
architectures such as Neural Turing Machine (NTM, \cite{graves2014neural})
and Differentiable Neural Computer (DNC, \cite{graves2016hybrid})
introduce more sophisticated memory writing:
\begin{equation}
M_{t}=M_{t-1}\odot\left(\boldsymbol{1}-w\left(M_{t-1},x_{t}\right)\otimes e\left(M_{t-1},x_{t}\right)\right)+w\left(M_{t-1},x_{t}\right)\otimes v\left(M_{t-1},x_{t}\right)
\end{equation}
where $w$, $e$ and $v$ are non-linear functions that take the previous
memory and the current input data as the input to produce the writing
weight, erase and value vectors, respectively. $\odot$ is the Hadamard
(element-wise) product. 

The problem with non-linear $f$ w.r.t $M$ is that the computation
must be done in recursive way, and thus being slow. Therefore, recent
memory-based models adopt simplified linear dynamics. For example,
Linear Transformer's memory update reads \cite{katharopoulos2020transformers}:
\begin{equation}
M_{t}=M_{t-1}+v\left(x_{t}\right)\otimes\frac{\phi\left(k\left(x_{t}\right)\right)}{\sum_{t}\phi\left(k\left(x_{t}\right)\right)}
\end{equation}
where $\phi$ is an activation function; $k$ and $v$ are functions
that transform the input to key and value. Recently, \citet{beck2024xlstm}
have proposed matrix-based LSTM (mLSTM):

\begin{equation}
M_{t}=\mathtt{f}\left(x_{t}\right)M_{t-1}+\mathtt{i}\left(x_{t}\right)v\left(x_{t}\right)\otimes k\left(x_{t}\right)
\end{equation}
where $\mathtt{f}$ and $\mathtt{i}$ are the forget and input gates,
respectively. In another perspective inspired by neuroscience, Fast
Forgetful Memory (FFM, \cite{morad2024reinforcement}) employs a parallelable
memory writing, which processes a single step update as follows:

\begin{equation}
M_{t}=M_{t-1}\odot\gamma+\left(v\left(x_{t}\right)\otimes\boldsymbol{1}^{\top}\right)\odot\gamma^{t-n}\label{eq:ffm}
\end{equation}
where $\gamma$ is a trainable matrix, $v$ is an input transformation
function, and $n$ is the last timestep.

\section{Methods}

In this section, we begin by introducing a unified memory writing
framework that incorporates several of the memory writing approaches
discussed earlier. Next, we examine the limitations of current memory
writing approaches through an analysis of this framework. Following
this analysis, we propose specialized techniques to address these
limitations. For clarity and consistency, all matrix and vector indices
will be referenced starting from 1, rather than 0. Constant matrices
are denoted by bold numbers.

\subsection{Hadamard Memory Framework (HMF)}

We propose a general memory framework that uses the Hadamard product
as its core operation. The memory writing at time step $t$ is defined
as:

\begin{equation}\label{eq:linm}   
\tcbox[nobeforeafter]{$M_{t}=M_{t-1}\odot\underbrace{C_{\theta}\left(x_{t}\right)}_{C_{t}}+\underbrace{U_{\varphi}\left(x_{t}\right)}_{U_{t}}$} 
\end{equation}where $C_{\theta}:\mathbb{R}^{D}\mapsto\mathbb{R}^{H\times H}$and
$U_{\varphi}:\mathbb{R}^{D}\mapsto\mathbb{R}^{H\times H}$ are parameterized
functions that map the current input $x_{t}$ to two matrices $C_{t}$
(calibration matrix) and $U_{t}$ (update matrix). Here, $\theta$
and $\varphi$ are referred to as calibration and update parameters.
Intuitively, the calibration matrix $C_{t}$ determines which parts
of the previous memory $M_{t-1}$ should be weakened and which should
be strengthened while the update matrix $U_{t}$ specifies the content
to be encoded into the memory. We specifically choose the Hadamard
product ($\odot$) as the matrix operator because it operates on each
memory element individually. We avoid using the matrix product to
prevent mixing the content of different memory cells during calibration
and update. Additionally, the matrix product is computationally slower. 

There are many ways to design $C_{t}$ and $U_{t}$. Given proper
choices of $C_{t}$ and $U_{t}$, Eqs. \ref{eq:fw}-\ref{eq:ffm}
can be reformulated into Eq. \ref{eq:linm}. Inspired by prior ``fast
weight'' works, we propose a simple update matrix:

\begin{equation}\label{eqn:ut}   
\tcbox[nobeforeafter]{$U_{\varphi}\left(x_{t}\right)=\eta_{\varphi}\left(x_{t}\right)\left[v\left(x_{t}\right)\otimes k\left(x_{t}\right)\right]$} 
\end{equation}where $k$ and $v$ are trainable neural networks that transform the
input $x_{t}$ to key and value representations. $\eta_{\varphi}:\mathbb{R}^{D}\mapsto\mathbb{R}$
is a parameterized function that maps the current input $x_{t}$ to
an update gate that controls the amount of update at step $t$. For
example, if $\eta_{\varphi}\left(x_{t}\right)=0$, the memory will
not be updated with any new content, whereas $\eta_{\varphi}\left(x_{t}\right)=1$
the memory will be fully updated with the content from the $t$-th
input. We implement $\eta_{\varphi}\left(x_{t}\right)$ as a neural
network with $\mathrm{sigmoid}$ activation function. 

We now direct our focus towards the design of $C_{t}$, which is the
core contribution of our work. The calibration matrix selectively
updates the memory by erasing no longer important memories and reinforcing
ongoing critical ones. In a degenerate case, if $C_{t}=\boldsymbol{1}$
for all $t$, the memory will not forget or strengthen any past information,
and will only memorize new information over time, similar to the Hebbian
Rule and Transformers. To analyze the role of the calibration matrix,
it is useful to unroll the recurrence, leading to the closed-form
equation (see proof in Appendix \ref{subsec:Close-form-Memory-Update}):

\begin{equation}
M_{t}=M_{0}\prod_{i=1}^{t}C_{i}+\sum_{i=1}^{t}U_{i}\odot\prod_{j=i+1}^{t}C_{j}\label{eq:unroll}
\end{equation}
where $\overset{}{\prod}$ represents element-wise products. Then,
$h_{t}=M_{t}q\left(x_{t}\right)=M_{0}\prod_{i=1}^{t}C_{i}q\left(x_{t}\right)+\sum_{i=1}^{t}U_{i}\odot\prod_{j=i+1}^{t}C_{j}q\left(x_{t}\right)$.
Calibrating the memory is important because without calibration ($C_{t}=\boldsymbol{1}$
$\forall t$), the read value becomes: $h_{t}=M_{0}q\left(x_{t}\right)+\sum_{i=1}^{t}U_{i}q\left(x_{t}\right).$
In this case, making $h_{t}$ to reflect a past context at any step
$j$ requires that $q\left(x_{t}\right)\neq0$ and $\sum_{i\neq j}U_{i}q\left(x_{t}\right)=\sum_{i\neq j}\eta_{\varphi}\left(x_{i}\right)v\left(x_{i}\right)\left[k\left(x_{i}\right)\cdot q\left(x_{t}\right)\right]\approx\boldsymbol{0}$,
which can be achieved if we can find $q\left(x_{t}\right)$ such that
$k\left(x_{i}\right)\cdot q\left(x_{t}\right)\approx0\,\,\forall i\neq j$.
Yet, this becomes hard when $T\gg H$ and $\eta_{\varphi}\left(x_{i}\right)\neq\boldsymbol{0}$
as it leads to an overdetermined system with more equations than variables.
We note that avoiding memorizing any $i$-th step with $\eta_{\varphi}\left(x_{i}\right)=0$
is suboptimal since $x_{i}$ may be required for another reading step
$t'\neq t$.

Therefore, at a certain timestep $t$, it is critical to eliminate
no longer relevant timesteps from $M_{t}$ by calibration, i.e., $U_{i}\odot\prod_{j=i+1}^{t}C_{j}\approx\boldsymbol{0}$
for unimportant $i$ (forgetting). For example, an agent may first
encounter an important event, like seeing a color code, before doing
less important tasks, such as picking apples. When reaching the goal
requiring to identify a door matching the color code, it would be
beneficial for the agent to erase memories related to the apple-picking
task, ensuring a clean retrieval of relevant information--the color
code. Conversely, if timestep $i$ becomes relevant again at a later
timestep $t'$, we need to recover its information, ensuring $U_{i}\odot\prod_{j=i+1}^{t'}C_{j}\neq\boldsymbol{0}$
(strengthening), just like the agent, after identifying the door,
may return to collecting apple task.
\begin{rem}
In the Hadamard Memory Framework, calibration should be enabled ($C_{t}\neq\boldsymbol{1}$)
and conditioned on the input context.
\end{rem}

Regarding computing efficiency, if $C_{t}$ and $U_{t}$ are not functions
of $M_{<t}$, we can compute the memory using Eq. \ref{eq:unroll}
in parallel, ensuring fast execution. In particular, the set of products
$\left\{ \prod_{j=i+1}^{t}C_{j}\right\} _{i=1}^{t}$ can be calculated
in parallel in $O\left(\log t\right)$ time. The summation can also
be done in parallel in $O\left(\log t\right)$ time. Additionally,
since all operations are element-wise, they can be executed in parallel
with respect to the memory dimensions. Consequently, the total time
complexity is $O\left(\log t\right)$. Appendix Algo. \ref{algo}
illustrates an implementation supporting parallelization. 
\begin{rem}
In the Hadamard Memory Framework, with optimal parallel implementation,
the time complexity for processing a sequence of $t$ steps is $O\left(\log t\right)$.
By contrast, without parallelization, the time complexity is $O\left(tH^{2}\right)$.
\end{rem}

\subsection{Challenges on Memory Calibration }

The calibration matrix enables agents to either forget or enhance
past memories. However, it complicates learning due to the well-known
issues of gradient vanishing or exploding. This can be observed when
examining the policy gradient over $T$ steps, which reads:

\begin{align}
\nabla_{\Theta}J\left(\pi_{\Theta}\right) & =\mathbb{E}_{s,a\sim\pi_{\Theta}}\sum_{t=0}^{T}\underbrace{\nabla_{\Theta}\log\pi_{\Theta}\left(a_{t}|M_{t}\right)}_{G_{t}\left(\Theta\right)}Adv\left(s_{t},a_{t},\gamma\right)
\end{align}
where $\Theta$ is the set of parameters, containing $\left\{ \theta,\varphi\right\} $,
$Adv$ represents the advantage function, which integrates reward
information $R\left(s_{t},a_{t}\right)$, and $G_{t}\left(\Theta\right)=\frac{\partial\log\pi_{\Theta}\left(a_{t}|M_{t}\right)}{\partial M_{t}}\frac{\partial M_{t}}{\partial\Theta}$
captures information related to the memory. Considering main gradient
at step $t$, required to learn $\theta$ and $\varphi$:

\begin{align}
\frac{\partial M_{t}}{\partial\theta} & =M_{0}\frac{\partial\prod_{i=1}^{t}C_{\theta}(x_{i})}{\partial\theta}+\sum_{i=1}^{t}U_{\varphi}\left(x_{i}\right)\odot\underbrace{\frac{\partial\prod_{j=i+1}^{t}C_{\theta}(x_{j})}{\partial\theta}}_{G_{1}\left(i,t,\theta\right)};\label{eq:gradM}\\
\frac{\partial M_{t}}{\partial\varphi} & =\sum_{i=1}^{t}\frac{\partial U_{\varphi}\left(x_{i}\right)}{\partial\varphi}\odot\underbrace{\prod_{j=i+1}^{t}C_{\theta}(x_{j})}_{G_{2}\left(i,t,\theta\right)}
\end{align}
We collectively refer $G_{1}\left(i,t,\theta\right)$ and $G_{2}\left(i,t,\theta\right)$
as $G_{1,2}\left(i,t,\theta\right)\in\mathbb{R}^{H\times H}$. These
terms are critical as they capture the influence of state information
at timestep $i$ on the learning parameters $\theta$ and $\varphi$.
The training challenges arise from these two terms as the number of
timesteps $t$ increases: (i) \emph{Numerical Instability (Gradient
Exploding):} if $\exists m.k\in\left[1,H\right]\,\mathrm{s.t.}\,G_{1,2}\left(i,t,\theta\right)\left[m,k\right]\rightarrow\infty$,
this leads to overflow, causing the gradient to become ``nan'';
(ii) \emph{Learning Difficulty (Gradient Vanishing)}: if $t\gg i_{0}$,
$\left\Vert G_{1,2}\left(i,t,\theta\right)\right\Vert \approx0\,\,\forall i<i_{0}$,
meaning no signal from timesteps $i<i_{0}$ contributes to learning
the parameters. This is suboptimal, especially when important observations
occur early in the episode, and rewards are sparse and given at the
episode end, .i.e., $R\left(s_{t},a_{t}\right)=0\text{\ensuremath{\,\forall\,t\neq T}}$.

\emph{How to design the calibration matrix $C_{\theta}\left(x\right)$
to overcome the training challenges?} A common approach is to either
fix it as hyperparameters or make it learnable parameters independent
on the input $x_{t}$ (e.g., Eqs. \ref{eq:fw} and \ref{eq:ffm}).
Unfortunately, we can demonstrate that this leads to either numerical
instability or learning difficulties as formalized in Proposition
\ref{prop3}. In the next section, we will provide a better design
for the calibration matrix.
\begin{prop}
If calibration is enabled $C_{\theta}\left(x_{t}\right)\neq\boldsymbol{1}$,
yet the calibration matrix is fixed, independent of the input $x_{t}$
($\forall t:C_{\theta}\left(x_{t}\right)=\theta\in\mathbb{R}^{H\times H}$),
numerical instability or learning difficulty will arise.\label{prop3}
\end{prop}

\begin{proof}
See Appendix \ref{subsec:Proposition-3}
\end{proof}

\subsection{Stable Hadamrad Memory (SHM)}

To avoid numerical and learning problems, it is important to ensure
each element of $C_{t}$ is not always greater than 1 or smaller than
1, which ends up in their product will be bounded such that $\text{\ensuremath{\mathbb{E}}}\left[\prod_{t=1}^{T}C_{t}\right]\neq\left\{ \boldsymbol{0},\boldsymbol{\infty}\right\} $
as $T\rightarrow\infty$. At the same time, we want $C_{t}$ to be
a function of $x_{t}$ to enable calibration conditioned on the current
context. To this end, we propose the calibration matrix:

\begin{equation}\label{eq:Ct}   
\tcbox[nobeforeafter]{$C_{\theta}\left(x_{t}\right)=1+\tanh\left(\theta_{t}\otimes v_{c}\left(x_{t}\right)\right)$} 
\end{equation}where $v_{c}:\mathbb{R}^{D}\mapsto\mathbb{R}^{H}$ is a mapping function,
and $\theta_{t}\in\mathbb{R}^{H}$ represents the main calibration
parameters. Here, we implement $v_{c}$ as a linear transformation
to map the input to memory space. Notably, the choice of $\theta_{t}$
determines the stability of the calibration. \emph{We propose to design
$\theta_{t}$ as trainable parameters that is randomly selected from
a set of parameters $\theta$}. In particular, given $\theta\in\mathbb{R}^{L\times H}$,
we sample uniformly a random row from $\theta$, $\theta_{t}=\theta\left[l_{t}\right]$
where $l_{t}\sim\mathcal{U}\left(1,L\right)$ where $L$ is the number
of possible $\theta_{t}$. We name the design as Stable Hadamard Memory
(SHM). Given the formulation, the range of an element $z_{t}^{m,k}=C_{\theta}\left(x_{t}\right)\left[m,k\right]$
is $\left[0,2\right]$ where $m,k\in\left[1,H\right]$. We can show
one important property of this choice is that $\text{\ensuremath{\mathbb{E}}}\left[\prod_{t=1}^{T}C_{t}\right]\approx\boldsymbol{1}$
under certain conditions. 
\begin{prop}
Assume: (i) $x_{t}\sim\mathcal{N}\left(0,\Sigma_{t}\right)$, (ii)
$v_{f}$ is a linear transformation, i.e.,$v_{c}\left(x_{t}\right)=Wx_{t}$,
(iii) $\left\{ z_{t}=\theta_{t}\otimes v_{c}\left(x_{t}\right)\right\} _{t=1}^{T}$
are independent across $t$. Given $C_{\theta}\left(x_{t}\right)$
defined in Eq. \ref{eq:Ct} then $\forall T\geq0,1\leq m,k\leq H$:

\[
\text{\ensuremath{\mathbb{E}}}\left[\prod_{t=1}^{T}z_{t}^{m,k}\right]=\text{\ensuremath{\mathbb{E}}}\left[\prod_{t=1}^{T}C_{\theta}\left(x_{t}\right)\left[m,k\right]\right]=1
\]
\end{prop}

\begin{proof}
see Appendix \ref{subsec:Proposition-4}
\end{proof}
Assumption (i) is reasonable as $x_{t}$ can be treated as being drawn
from a Gaussian distribution, and LayerNorm can be applied to normalize
$x_{t}$ to have a mean of zero. Assumption (ii) can be realized as
we implement $v_{c}$ as a linear transformation. Assumption (iii)
is more restrictive because $\left\{ x_{t}\right\} _{t=1}^{T}$ are
often dependent in RL setting, which means $\left\{ z_{t}=\theta_{t}\otimes v_{c}\left(x_{t}\right)\right\} _{t=1}^{T}$
are not independent and thus, $\text{\ensuremath{\mathbb{E}}}\left[\prod_{t=1}^{T}z_{t}^{m,k}\right]\neq1$.
However, by reducing the linear dependence between $z_{t}$ through
random selection of $\theta_{t}$, we can make $\text{\ensuremath{\mathbb{E}}}\left[\prod_{t=1}^{T}z_{t}^{m,k}\right]$
closer to $1$, and thus being bounded. Specifically, we will prove
that by using $\theta_{t}=\theta\left[l_{t}\right];l_{t}\sim\mathcal{U}\left(1,L\right)$
the Pearson Correlation Coefficient between timesteps is minimized,
as stated in the following proposition:
\begin{prop}
Let $z_{t}^{m,k}=u_{t}^{m}v_{t}^{k}$ where $z_{t}^{m,k}=\left(\theta_{t}\otimes v_{c}\left(x_{t}\right)\right)\left[m,k\right]$,
$u_{t}^{m}=\theta_{t}\left[m\right]$ and $v_{t}^{k}=v_{c}\left(x_{t}\right)\left[k\right]$.
Given the Pearson correlation coefficient of two random variables
$X$ and $Y$ is defined as $\rho\left(X,Y\right)=\frac{\text{Cov}(X,Y)}{\sqrt{\text{Var}(X)}\sqrt{\text{Var}(Y)}}$,
then $\forall v_{t}^{k},v_{t'}^{k}$:
\[
\left|\rho\left(u_{t}^{m}v_{t}^{k},u_{t'}^{m}v_{t'}^{k}\right)\right|\leq\left|\rho\left(v_{t}^{k},v_{t'}^{k}\right)\right|
\]
The equality holds when $u_{t}^{m}=\beta u_{t'}^{m}$.\label{prop:Le5}
\end{prop}

\begin{proof}
See Appendix \ref{subsec:Proposition-5}
\end{proof}
As a result, our choice of $\theta_{t}$ outperforms other straightforward
designs for minimizing dependencies between timesteps. For instance,
a fixed $\theta_{t}$ ($\theta_{t}=\theta\in\mathbb{R}^{H}$) results
in higher dependencies because $\rho(\theta\left[m\right]v_{t}^{k},\theta\left[m\right]v_{t'}^{k})=\rho(v_{t}^{k},v_{t'}^{k})\geq\rho\left(u_{t}^{m}v_{t}^{k},u_{t'}^{m}v_{t'}^{k}\right)$.
In practice, even when $\text{\ensuremath{\mathbb{E}}}\left[\prod_{t=1}^{T}z_{t}^{m,k}\right]$
is bounded, the cumulative product can occasionally become very large
for certain episodes and timesteps, leading to overflow and disrupting
the learning process. This can be avoided by clipping the gradients.
In experiments, we implement SHM using nonparallel recursive form
(Eq. \ref{eq:linm}). The memory is then integrated into policy-gradient
RL algorithms to optimize Eq. \ref{eq:rlobj}, with the read value
$h_{t}$ (Eq. \ref{eq:hread}) used as input for value/policy networks.
Code will be available upon publication.

\section{Experimental Results}

We evaluate our method alongside notable memory-augmented agents in
POMDP environments. Unless stated otherwise, the context consists
of the observation and previous action. All training uses the same
hardware (single NVDIA H100 GPU), RL architecture, algorithm, training
protocol, and hyperparameters. We focus only on model-free RL algorithms.
The baselines differ only in their memory components: GRU \cite{chung2014empirical},
FWP \cite{schlag2021linear}, GPT-2 \cite{radford2019language}, S6
\cite{gu2023mamba}, mLSTM \cite{beck2024xlstm}, FFM \cite{morad2024reinforcement}
and SHM (Ours). For tasks with a clear goal, we measure performance
using the Success Rate, defined as the ratio of episodes that reach
the goal to the total number of evaluation episodes. For tasks in
POP-Gym, we use Episode Return as the evaluation metric. We fix SHM's
number of possible $\theta_{t}$, $L=128$, across experiments. 

\subsection{Sample-Efficient Meta Reinforcement Learning\label{subsec:Meta-Reinforcement-Learning}}

Meta-RL targets POMDPs where rewards and environmental dynamics differ
across episodes, representing various tasks \cite{schmidhuber1987evolutionary,thrun1998learning}.
To excel in all tasks, memory agents must learn general memory update
rules that can adapt to any environments. We enhanced the Wind and
Point Robot environments from \citet{ni2022recurrent} to increase
difficulty. In these environments, the observation consists of the
agent's 2D position $p_{t}$, while the goal state $p_{g}$ is hidden.
The agent takes continuous actions $a_{t}$ by moving with 2D velocity
vector. The sparse reward is defined as $R(p_{t+1},p_{g})=\boldsymbol{1}(\|p_{t+1}-p_{g}\|_{2}\leq r)$
where $r=0.01$ is the radius. In Wind, the goal is fix $p_{g}=\left[0,1\right]$,
yet there are noises in the dynamics: $\ensuremath{p_{t+1}=p_{t}+a_{t}+w},$
with the ``wind'' $w$ sampled from $U[-0.1,0.1]$ at the start
and fixed thereafter. In Point Robot, the goal varies across episodes,
sampled from $U[-10,10]$. To simulate real-world conditions where
the training tasks are limited, we create 2 modes using different
number of training and testing tasks: $\left[50,150\right]$ and $\left[10,190\right]$,
respectively. Following the modifications, these simple environments
become significantly more challenging to navigate toward the goal,
so we set the episode length to 100.

We incorporate the memory methods to the Soft Actor Critic (SAC, \cite{haarnoja2018soft}),
using the same code base introduced in \citet{ni2022recurrent}. We
keep the SHM model sizes and memory capacities small, at around 2MB
for the checkpoint and 512 memory elements, which is roughly equivalent
to a GRU (see Appendix \ref{subsec:Sample-Efficiency-in-Meta}). We
train all models for 500,000 environment steps, and report the learning
curves in Fig. \ref{fig:Visual-Match-and-1}. In the easy mode (50-150),
our SHM method consistently achieves the best performance, with a
near-optimal success rate, while other methods underperform by 20-50\%
on average in Wind and Point Robot, respectively. In the hard mode
(10-190), SHM continues to outperform other methods by approximately
20\%, showing earlier signs of learning. 

\begin{figure}
\begin{centering}
\includegraphics[width=1\textwidth]{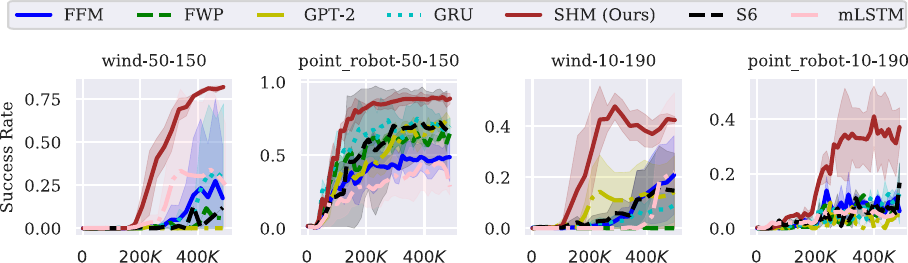}
\par\end{centering}
\caption{Meta-RL: Wind and Point Robot learning curves. Mean $\pm$ std. over
5 runs. \label{fig:Visual-Match-and-1} }
\end{figure}
\begin{figure}
\begin{centering}
\includegraphics[width=1\textwidth]{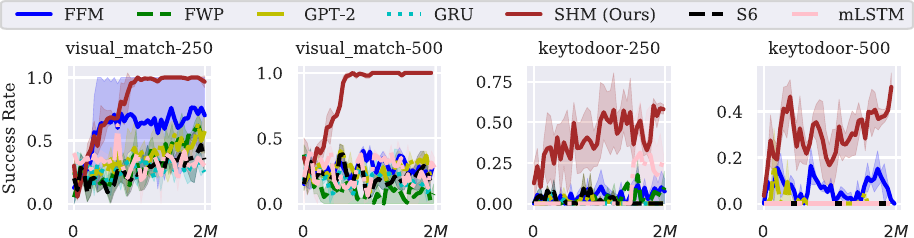}
\par\end{centering}
\caption{Credit Assignment: Visual Match, Key-to-Door learning curves. Mean
$\pm$ std. over 3 runs. \label{fig:Visual-Match-and} }
\end{figure}

\subsection{Long-term Credit Assignment}

In this task, we select the Visual Match and Key-to-Door environments,
the most challenge tasks mentioned in \citet{ni2024transformers}.
Both have observation as the local view of the agent, discrete actions
and sparse rewards dependent on the full trajectory, requiring long-term
memorization. In particular, the pixel-based Visual Match task features
an intricate reward system: in Phase 1, observing color codes yields
no rewards, while in Phase 2, picking an apple provides immediate
reward of one, relying on short-term memory. The final reward--a
bonus for reaching the door with the matching color code is set to
10. Key-to-Door also involves multiple phases: finding the key, picking
apples, and reaching the door. The terminal reward is given if the
key was acquired in the first phase and used to open the door. Both
tasks can be seen as decomposed episodic problems with noisy immediate
rewards, requiring that in the final phase, the agent remembers the
event in the first phase. We create 2 configurations using different
episode steps in the second phases: 250 and 500, respectively. 

Still following \citet{ni2022recurrent}, we use SAC-Discrete \cite{christodoulou2019soft}
as the RL algorithm. We use the same set of baselines as in Sec. \ref{subsec:Meta-Reinforcement-Learning}
and train them for 2 million environment steps. The results in Fig.
\ref{fig:Visual-Match-and} clearly show that, our method, SHM, significantly
outperforms all other methods in terms of success rate. Notably, SHM
is the only method that can perfectly solve both 250 and 500-step
Visual Match while the second-best performer, FFM, achieves only a
77\% and 25\% success rate, respectively. In Key-To-Door, our method
continues showing superior results with high success rate. By contrast,
no meaningful learning is observed from the other methods, which perform
similarly to GPT-2, as also noted by \citet{ni2024transformers}. 

\subsection{POPGym Hardest Games}

We evaluate SHM on the POPGym benchmark \cite{morad2023popgym}, the
largest POMDP benchmark to date. Following previous studies \cite{morad2023popgym,samsami2024mastering},
we focus on the most memory-intensive tasks: Autoencode, Battleship,
Concentration and RepeatPrevious. These tasks require ultra long-term
memorization, with complexity increasing across Easy, Medium, and
Hard levels. All tasks use categorical action and observation spaces,
allowing up to 1024 steps. 

For comparison, we evaluate SHM against \emph{state-of-the-art model-free
methods}, including GRU and FFM. Other memory models, such as DNC,
Transformers, FWP, and SSMs, have been reported to perform worse.
The examined memory models are integrated into PPO \cite{schulman2017proximal},
trained for 15 million steps using the same codebase as  \citet{morad2023popgym}
to ensure fairness. The models differ only in their memory, controlled
by the memory dimension hyperparameter $H$. We tune it for each baseline,
adjusting it to optimize performance, as larger $H$ values typically
improve results. The best-performing configurations are reported in
Table \ref{tab:PopGym:-Average-return}, where SHM demonstrates a
relative improvement of $\approx$10-12\% over FFM and GRU on average.
Notably, the learning curves in Appendix Fig. \ref{fig:POPGym-learning-curves:}
show that only SHM demonstrates signs of learning in several tasks,
including RepeatPrevious-Medium/Hard, and Autoencode-Easy/Medium.
Detailed hyperparameter setting and additional results are provided
in Appendix \ref{subsec:POPGym-Hardest-GamesDt}.

In terms of running time, the average batch inference time in milliseconds
for GRU, FFM, and SHM is 1.6, 1.8, and 1.9, respectively, leading
to a total of 7, 8, and 9 hours of training per task. While SHM is
slightly slower than GRU and FFM, the difference is a reasonable trade-off
for improved memory management in partially observable environments.
Last but not least, our model's runtime was measured using a non-parallel
implementation, while GRU benefits from hardware optimization in the
PyTorch library. SHM's running time could be further improved with
proper parallelization.

\begin{table}
\begin{centering}
\begin{tabular}{ccccc}
\hline 
\multirow{1}{*}{{\small{}Task}} & {\small{}Level} & \multicolumn{1}{c}{{\small{}GRU}} & \multicolumn{1}{c}{{\small{}FFM}} & \multicolumn{1}{c}{{\small{}SHM (Ours)}}\tabularnewline
\hline 
\multirow{3}{*}{{\small{}Autoencode}} & {\small{}Easy} & {\small{}-37.9$\pm$7.7} & {\small{}-32.7$\pm$0.6} & \textbf{\small{}49.5$\pm$23.3}\tabularnewline
 & {\small{}Medium} & {\small{}-43.6$\pm$3.5} & {\small{}-32.7$\pm$0.6} & \textbf{\small{}-28.8$\pm$14.4}\tabularnewline
 & {\small{}Hard} & {\small{}-48.1$\pm$0.7} & {\small{}-47.7$\pm$0.5} & \textbf{\small{}-43.9$\pm$0.9}\tabularnewline
\hline 
\multirow{3}{*}{{\small{}Battleship}} & {\small{}Easy} & {\small{}-41.1$\pm$1.0} & {\small{}-34.0$\pm$7.1} & \textbf{\small{}-12.3$\pm$2.4}\tabularnewline
 & {\small{}Medium} & {\small{}-39.4$\pm$0.5} & {\small{}-37.1$\pm$3.1} & \textbf{\small{}-16.8$\pm$0.6}\tabularnewline
 & {\small{}Hard} & {\small{}-38.5$\pm$0.5} & {\small{}-38.8$\pm$0.3} & \textbf{\small{}-21.2$\pm$2.3}\tabularnewline
\hline 
\multirow{3}{*}{{\small{}Concentration}} & {\small{}Easy} & {\small{}-10.9$\pm$1.0} & \textbf{\small{}10.7$\pm$1.2} & {\small{}-1.9$\pm$2.4}\tabularnewline
 & {\small{}Medium} & {\small{}-21.4$\pm$0.5} & {\small{}-24.7$\pm$0.1} & \textbf{\small{}-21.0$\pm$0.8}\tabularnewline
 & {\small{}Hard} & {\small{}-84.0$\pm$0.3} & {\small{}-87.5$\pm$0.5} & \textbf{\small{}-83.3$\pm$0.1}\tabularnewline
\hline 
\multirow{3}{*}{{\small{}RepeatPrevious}} & {\small{}Easy} & \textbf{\small{}99.9$\pm$0.0} & {\small{}98.4$\pm$0.3} & {\small{}88.9$\pm$11.1}\tabularnewline
 & {\small{}Medium} & {\small{}-34.7$\pm$1.7} & {\small{}-24.3$\pm$0.4} & \textbf{\small{}48.2$\pm$7.2}\tabularnewline
 & {\small{}Hard} & {\small{}-41.7$\pm$1.8} & {\small{}-33.9$\pm$1.0} & \textbf{\small{}-19.4$\pm$9.9}\tabularnewline
\hline 
{\small{}Average} & {\small{}All} & {\small{}-28.4$\pm$1.3} & {\small{}-24.2$\pm$1.2} & \textbf{\small{}-5.1$\pm$6.3}\tabularnewline
\hline 
\end{tabular}
\par\end{centering}
\caption{PopGym: Mean return $\pm$ std. ($\times100$) at the end of training
over 3 runs. The range of return ($\times100$) is $\left[-100,100\right]$.\label{tab:PopGym:-Average-return}}

\end{table}

\subsection{Model Analysis and Ablation Study\label{subsec:Ablation-Studies}}

\textbf{Choice of Calibration} Prop. \ref{prop:Le5} suggests that
selecting random $\theta_{t}\in\mathbb{R}^{H}$ in Eq. \ref{eq:Ct}
will reduce the dependencies between $C_{t}$, bringing $\prod_{t=1}^{T}C_{t}$
closer to $\boldsymbol{1}$ to avoid gradient issues. In this section,
we empirically verify that by comparing our proposed Random $\theta_{t}$
with the following designs: $C=\boldsymbol{1}$, no calibration is
used; \emph{Random $C$}, where a random calibration matrix is sampled
from normal distribution at each timestep, having $\text{\ensuremath{\mathbb{E}}}\left[\prod_{t=1}^{T}z_{t}^{m,k}\right]=1$
under mild assumptions, but preventing learning meaningful calibration
rules;\emph{ Fixed $C$}, a simpler method for learning calibration,
but prone to gradient problems (Prop. \ref{prop3}); \emph{Fixed $\theta_{t}$},
where we learn fixed parameter $\theta_{t}$, which is proven to be
less effective than Random $\theta_{t}$ in reducing linear dependencies
between timesteps (Prop. \ref{prop:Le5}); Neural $\theta_{t}$, where
$\theta_{t}=FFW\left(x_{t}\right)$, generated by a feedforward neural
network like mLSTM, but with no guarantee of reducing timestep dependencies.
We trained RL agents using the above designs of the calibration matrix
with $H=72$ on the Autoencode-Easy task and present the learning
curves in Fig. \ref{fig:(a)-Performance-of} (a, left). The results
show that our proposed Random $\theta_{t}$ outperforms the other
baselines by a substantial margin, with at least a 30\% improvement
in performance. This confirms the effectiveness of our calibration
design in enhancing the final results.

\begin{figure}
\begin{centering}
\includegraphics[width=1\textwidth]{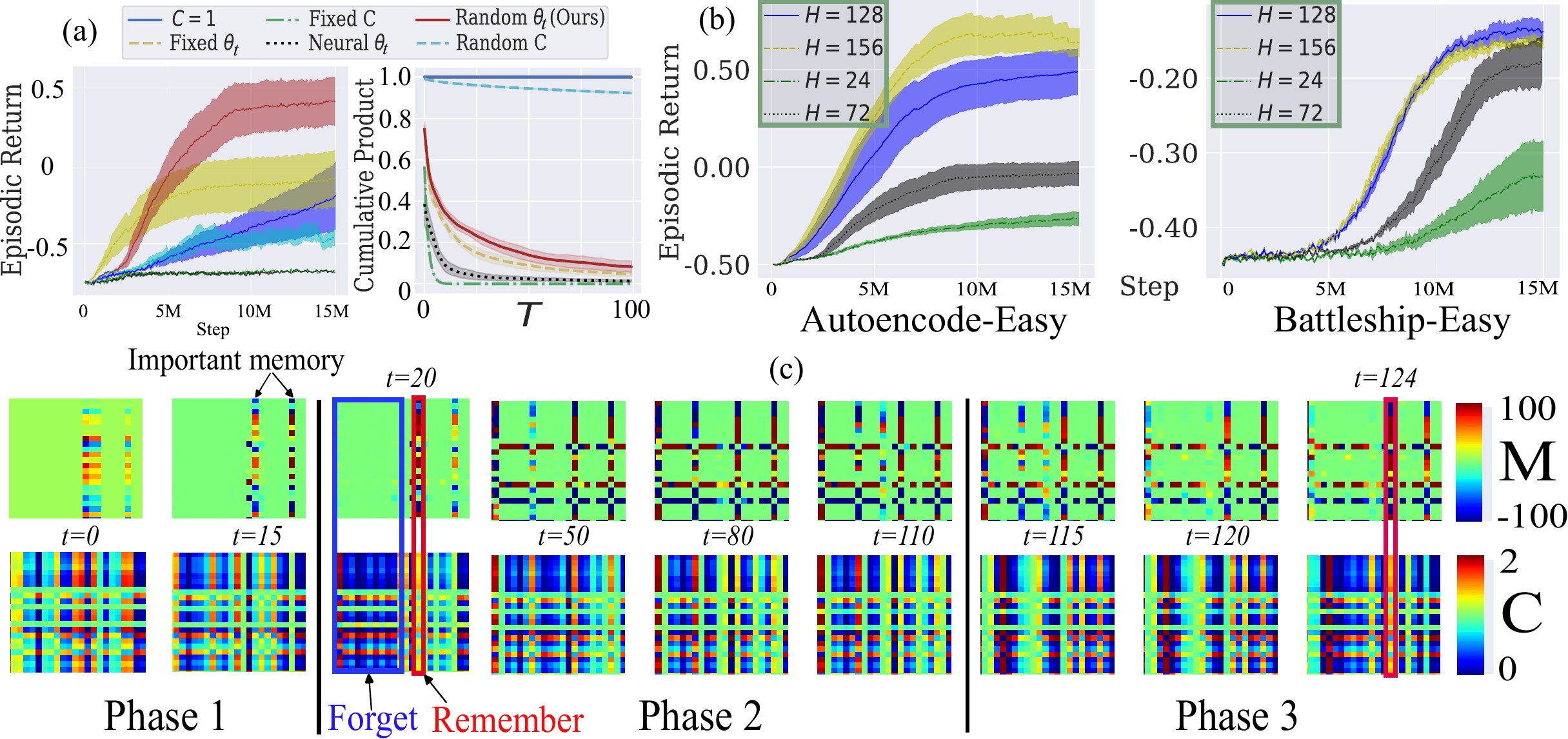}
\par\end{centering}
\caption{(a) Left: Return of calibration designs over 3 runs; Right: Calibration
matrix cumulative product over 100 episodes. (b) Return of memory
sizes $H$ on Autoencode-Easy (left) and Battleship-Easy (right).
(c) Memory ($M$, top) and calibration ($C$, bottom) matrices over
timesteps in Visual Match: SHM erases memory that are no longer required
and strengthens the important ones. \label{fig:(a)-Performance-of}}
\end{figure}
\textbf{Vanishing Behavior} In practice, exploding gradients can be
mitigated by clipping. Thus, we focus on the vanishing gradient, which
depends on the cumulative product $\mathcal{C}_{j}$ $=\prod_{t=1}^{j}C_{t}$.
Our theory suggests that Random $\theta_{t}$ should be less susceptible
to the vanishing phenomenon compared to other methods such as Fixed
$C$, Fixed $\theta_{t}$ and Neural $\theta_{t}$. To verify that,
for each episode, we compute the average value of elements in the
matrix $\mathcal{C}_{j}$ that are smaller than 1 ($\overline{\mathcal{C}_{j}}\left[<1\right]$),
as those larger than 1 are prone to exploding and are not appropriate
for averaging with the vanishing values. We plot $\overline{\mathcal{C}_{j}}\left[<1\right]$
for $j=1,2,...100$ over 100 episodes in Fig \ref{fig:(a)-Performance-of}
(a, right). 

The results are consistent with theoretical predictions: Fixed $C$
leads to rapid vanishing of the cumulative product in just 10 steps.
Neural $\theta_{t}$ also perform badly, likely due to more complex
dependencies between timesteps because $z_{t}^{m,k}$ now becomes
non-linear function of $x_{t}$, causing $\prod_{t=1}^{T}z_{t}^{m,k}$
to deviate further from 1. While Fixed $\theta_{t}$ is better than
Neural $\theta_{t}$, it still exhibits quicker vanishing compared
to our approach Random $\theta_{t}$. As expected, Random $C$ shows
little vanishing, but like setting $C=1$, it fails to leverage memory
calibration, resulting in underperformance (Fig. \ref{fig:(a)-Performance-of}
(a, left)). Random $\theta_{t}$, although its $\overline{\mathcal{C}_{j}}$
also deviates from $1$, shows the smallest deviation among the calibration
learning approaches. Additionally, the vanishing remain manageable
after 100 timesteps, allowing gradients to still propagate effectively
and enabling the calibration parameters to be learned.

\textbf{Memory Size} The primary hyperparameter of our method is $H$,
determining the memory capacity. We test SHM with $H\in\left\{ 24,72,128,156\right\} $
on the Autoencode-Easy and Battleship-Easy tasks. Fig. \ref{fig:(a)-Performance-of}
(b) shows that larger memory generally results in better performance.
In terms of speed, the average batch inference times (in milliseconds)
for different $H$ values are 1.7, 1.8, 1.9, and 2.1, respectively.
We choose $H=128$ for other POPGym tasks to balance performance and
speed. 

\textbf{Forgetting and Remembering }Here, we investigate the learned
calibration strategy of SHM on Visual Match with 15,100, and 10 steps
in Phase 1, 2 and, 3, respectively. We sample several representative
$M_{t}$ and $C_{t}$ from 3 phases and visualize them in Fig. \ref{fig:(a)-Performance-of}
(c). In Phase 1, the agent identifies the color code and stores it
in memory, possibly in two columns of $M$, marked as ``important
memory''. In Phase 2, unimportant memory elements are erased where
$C_{t}\approx0$ (e.g., those within the blue rectangle). However,
important experiences related to the Phase 1's code are preserved
across timesteps until Phase 3 (e.g., those within the red rectangle
where $C_{t}\gtrsim$1), which is desirable.

\section{Related works}

Numerous efforts have been made to equip RL agents with memory mechanisms.
Two main research directions focus on inter-episode and intra-episode
memory approaches. While episodic control with a global, tabular memory
enhances sample efficiency by storing experiences across episodes
\cite{blundell2016model,le2021model,le2022episodic}, it falls short
in recalling specific events within individual episodes. Similarly,
global memory mechanisms can support exploration or optimization \cite{badia2020never,le2022learning,10.5555/3635637.3662964},
but are not designed to address the challenges of memorization within
a single episode. In contrast, to address the challenges of partially
observable Markov decision processes (POMDPs), RL agents often leverage
differentiable memory, which is designed to capture the sequence of
observations within an episode, and can be learned by policy gradient
algorithms \cite{wayne2018unsupervised}.

Differentiable memory models can be broadly categorized into vector-based
and matrix-based approaches. Vector-based memory, like RNNs \cite{elman1990finding},
processes inputs sequentially and stores past inputs in their hidden
states. While RNNs are slower to train, they are efficient during
inference. Advanced variants, such as GRU and LSTM, have shown strong
performance in POMDPs, often outperforming more complex RL methods
\cite{ni2022recurrent,morad2023popgym}. Recently, faster alternatives
like convolutional and structured state space models (SSM) have gained
attention \cite{bai2018empirical,gu2020hippo}, though their effectiveness
in RL is still under exploration. Initial attempts with models like
S4 underperformed in POMDP tasks \cite{morad2023popgym}, but improved
SSM versions using S5, S6 or world models have shown promise \cite{lu2024structured,gu2023mamba,samsami2024mastering}.
Despite these advancements, vector-based memory is limited, as compressing
history into a single vector makes it challenging to scale for high-dimensional
memory space.

Matrix-based memory, on the other hand, offers higher capacity by
storing history in a matrix but at the cost of increased complexity.
Attention-based models, such as Transformers \cite{Vaswani2017AttentionIA},
have largely replaced RNNs in SL, also delivering good results in
standard POMDPs \cite{parisotto2020stabilizing}. However, their quadratic
memory requirements limit their use in environments with long episodes.
Empirical studies have also shown that Transformers struggle with
long-term memorization and credit assignment tasks \cite{ni2024transformers}.
While classic memory-augmented neural networks (MANNs) demonstrated
good performance in well-crafted long-term settings \cite{graves2016hybrid,hung2019optimizing,le2020self},
they are slow and do not scale well in larger benchmarks like POPGym
\cite{morad2023popgym}. New variants of LSTM \cite{beck2024xlstm},
including those based on matrices, have not been tested in reinforcement
learning settings and lack theoretical grounding to ensure stability.

Simplified matrix memory models \cite{katharopoulos2020transformers,schlag2021linear},
offer scalable solutions but have underperformed compared to simple
RNNs in the POPGym benchmark, highlighting the challenges of designing
efficient matrix memory for POMDPs. Recently, Fast and Forgetful Memory
(FFM, \cite{morad2024reinforcement}), incorporating inductive biases
from neuroscience, has demonstrated better average results than RNNs
in the benchmark. However, in the most memory-intensive environments,
the improvement remains limited. Compared to our approach, these matrix-based
memory methods lack a flexible memory calibration mechanism and do
not have robust safeguards to prevent numerical and learning issues
in extremely long episodes.

\section{Discussion}

In this paper, we introduced the Stable Hadamard Framework (SHF) and
its effective instance, the Stable Hadamard Memory (SHM), a novel
memory model designed to tackle the challenges of dynamic memory management
in partially observable environments. By utilizing the Hadamard product
for memory calibration and update, SHM provides an efficient and theoretically
grounded mechanism for selectively erasing and reinforcing memories
based on relevance. Our experiments on the POPGym and POMDP benchmarks
demonstrate that SHM significantly outperforms state-of-the-art memory-based
models, particularly in long-term memory tasks, while being fast to
execute. Although our theory suggests that SHM should be more stable
and mitigate gradient learning issues by reducing linear dependencies
between timesteps, this stability is not guaranteed to be perfect.
Further theoretical investigation is needed to validate and refine
these properties in future work.

\bibliographystyle{plainnat}
\bibliography{pma}

\cleardoublepage\newpage{}

\section*{Appendix}

\renewcommand\thesubsection{\Alph{subsection}}

\subsection{Theoretical Results}

\subsubsection{Closed-form Memory Update \label{subsec:Close-form-Memory-Update}}

Given 
\begin{equation}
M_{t}=M_{t-1}\odot C_{t}+U_{t}\label{eq:recursive}
\end{equation}
then

\begin{equation}
M_{t}=M_{0}\prod_{i=1}^{t}C_{i}+\sum_{i=1}^{t}U_{i}\odot\prod_{j=i+1}^{t}C_{j}\label{eq:unroll-1}
\end{equation}

\begin{proof}
We prove by induction. 

\textbf{Base case:} for $t=1$, the equation becomes for both Eqs.
\ref{eq:recursive} and \ref{eq:unroll-1}:
\[
M_{1}=M_{0}\cdot C_{1}+U_{1}
\]
 Thus, the equation holds for $t=1$.

\textbf{Inductive hypothesis:} Assume the equation holds for $t=n$:
\[
M_{n}=M_{0}\prod_{i=1}^{n}C_{i}+\sum_{i=1}^{n}U_{i}\odot\prod_{j=i+1}^{n}C_{j}
\]

\textbf{Inductive step:} We now prove the equation holds for $t=n+1$.
From the update rule in Eq. \ref{eq:recursive}: 
\[
M_{n+1}=M_{n}\odot C_{n+1}+U_{n+1}
\]
Substitute$M_{n}$ using the inductive hypothesis:
\begin{align*}
M_{n+1} & =\left(M_{0}\prod_{i=1}^{n}C_{i}+\sum_{i=1}^{n}U_{i}\odot\prod_{j=i+1}^{n}C_{j}\right)\odot C_{n+1}+U_{n+1}\\
 & =M_{0}\prod_{i=1}^{n+1}C_{i}+\sum_{i=1}^{n}U_{i}\odot\prod_{j=i+1}^{n+1}C_{j}+U_{n+1}\\
 & =M_{0}\prod_{i=1}^{n+1}C_{i}+\sum_{i=1}^{n+1}U_{i}\odot\prod_{j=i+1}^{n+1}C_{j}
\end{align*}
This matches the form of the closed-form Eq. \ref{eq:unroll-1} for
$t=n+1$, completing the proof by induction.
\end{proof}

\subsubsection{Proposition 3\label{subsec:Proposition-3}}
\begin{defn}
\textbf{Critical Memory Gradients}: In the Hadamard Memory Framework,
we define the critical memory gradients of memory rules as follows:
\end{defn}

\[
G_{1}\left(i,t,\theta\right)=\frac{\partial\prod_{j=i+1}^{t}C_{\theta}(x_{j})}{\partial\theta}=\sum_{j=i+1}^{t}\frac{\partial C_{\theta}\left(x_{j}\right)}{\partial\theta}\odot\prod_{k=i+1,k\neq j}^{t}C_{\theta}(x_{k})
\]

\[
G_{2}\left(i,t,\theta\right)=\prod_{j=i+1}^{t}C_{\theta}(x_{j})
\]
Now we can proceed to the proof for Proposition \ref{prop3}. 
\begin{proof}
In the case $C_{\theta}\left(x_{t}\right)=\theta\in\mathbb{R}^{H\times H}$,
the critical gradients read:
\begin{align*}
G_{1}\left(i,t,\theta\right) & =\sum_{j=i+1}^{t}\prod_{k=i+1,k\neq j}^{t}\theta\\
 & =\left(t-i\right)\theta^{t-i-1}
\end{align*}
\[
G_{2}\left(i,t,\theta\right)=\theta^{t-i}
\]
If $\exists\,0\leq m,k<H\,\mathrm{s.t.}\,\left|\theta\left[m,k\right]\right|>1$,
$G_{1,2}\left(i,t,\theta\right)\left[m,k\right]\sim o\left(\theta\left[m,k\right]^{t-i-1}\right)$,
and thus $\rightarrow\infty$, i.e., numerical problem arises. Let
$\left\Vert \cdot\right\Vert $ denote the infinity norm, we also
have:
\[
\left\Vert G_{1}\left(i,t,\theta\right)\right\Vert =\left\Vert \left(t-i\right)\theta^{t-i-1}\right\Vert \leq\left(t-i\right)\left\Vert \theta\right\Vert ^{t-i-1}
\]
\[
\left\Vert G_{2}\left(i,t,\theta\right)\right\Vert =\left\Vert \theta^{t-i}\right\Vert \leq\left\Vert \theta\right\Vert ^{t-i}
\]
Note that if $\forall m,k\left|\theta\left[m,k\right]\right|<1$ ,
$\left\Vert \theta\right\Vert <1$, both terms become 0 as $t-i$
increases, thus learning problem always arises. In conclusion, in
this case, to avoid both numerical and learning problems, $\forall m,k\left|\theta\left[m,k\right]\right|=1$,
which is not optimal in general.
\end{proof}

\subsubsection{Proposition 4\label{subsec:Proposition-4}}
\begin{proof}
Let $z_{t}^{m,k}=u_{t}^{m}v_{t}^{k}$ where $z_{t}^{m,k}=\left(\theta_{t}\otimes v_{c}\left(x_{t}\right)\right)\left[m,k\right]$,
$u_{t}^{m}=\theta_{t}\left[m\right]$ and $v_{t}^{k}=v_{c}\left(x_{t}\right)\left[k\right]$.
Using assumption (i) and (ii), $v_{t}^{k}$ is a Gaussian variable,
i.e., $v_{t}^{k}\sim\mathcal{N}\left(0,\mu_{t}^{k}\right)$. By definition,
$u_{t}^{m}$ is a categorical random variable that can take values
$\left\{ \theta\left[m,l_{t}\right]\right\} _{l_{t}=1}^{L}$ with
equal probability $1/L$. For now, we drop the subscripts $m$, $k$
and $t$ for notation ease. The PDF of $z=uv$ can be expressed as
a mixture distribution since $u$ is categorical and can take discrete
values $\left\{ u_{l}\right\} _{l=1}^{L}$. The PDF of $z$, denoted
as $f\left(z\right)$, is given by:

\begin{align*}
f(z) & =\sum_{l=1}^{L}P(u=u_{l})f_{u_{l}v}(z)\\
 & =\frac{1}{L}\sum_{l=1}^{L}f_{u_{l}v}(z)
\end{align*}
where $f_{u_{l}v}(z)$ is the PDF of $u_{l}v$, and $u_{l}$ is a
constant for each $l$. Thus, $f_{u_{l}v}(z)$ is the scaled PDF of
$u$:

\[
f_{u_{l}v}(z)=\frac{1}{|u_{l}|}f_{v}\left(\frac{z}{u_{l}}\right)
\]
Since $v\sim\mathcal{N}\left(0,\mu\right)$, the PDF of $v$, denoted
as $f_{v}(x)$, is symmetric about 0, we have:

\[
f_{u_{l}v}(z)=\frac{1}{|u_{l}|}f_{v}\left(\frac{z}{u_{l}}\right)=\frac{1}{|u_{l}|}f_{v}\left(\frac{-z}{u_{l}}\right)=f_{u_{l}v}(-z).
\]
This shows that $f_{u_{l}v}(z)$ is symmetric around 0 for each $l$.
Therefore, the PDF $f(z)$ is also symmetric:

\[
f(z)=\frac{1}{L}\sum_{l=1}^{L}f_{u_{l}v}(z)=\frac{1}{L}\sum_{l=1}^{L}f_{u_{l}v}(-z)=f(-z)
\]
Since $\tanh$ is an odd function and the PDF of $z^{m,k}$ is symmetric
about 0, $\mathbb{E}\left[\tanh\left(z^{m,k}\right)\right]=0$ and
thus $\mathbb{E}\left[1+\tanh\left(z^{m,k}\right)\right]=1$. Finally,
using assumption (iii), $\text{\ensuremath{\mathbb{E}}}\left[\prod_{t=1}^{T}z_{t}^{m,k}\right]=\prod_{t=1}^{T}\text{\ensuremath{\mathbb{E}}}\left[z_{t}^{m,k}\right]=1$.
\end{proof}

\subsubsection{Proposition 5\label{subsec:Proposition-5}}
\begin{proof}
Without loss of generalization, we can drop the indice $m$ and $k$
for notation ease. Since $\theta_{t}=\theta\left[l_{t}\right];l_{t}\sim\mathcal{U}\left(1,L\right)$,
it is reasonable to assume that each of $\left\{ u_{t},u_{t'}\right\} $,
$\left\{ u_{t},v_{t}\right\} $, $\left\{ u_{t},v_{t'}\right\} $,
$\left\{ u_{t'},v_{t}\right\} $, $\left\{ u_{t'},v_{t'}\right\} $
are independent. In this case, let us denote $X=u_{t}v_{t}$ and $X'=u_{t'}v_{t'}$,
we have

\begin{align*}
\text{Cov}(X,X') & =\text{Cov}(u_{t}v_{t},u_{t'}v_{t'})\\
 & =\mathbb{E}[u_{t}v_{t}\cdot u_{t'}v_{t'}]-\mathbb{E}[u_{t}v_{t}]\cdot\mathbb{E}[u_{t'}v_{t'}]\\
 & =\mathbb{E}(u_{t}v_{t})\mathbb{E}(u_{t'}v_{t'})-\mathbb{E}(u_{t})\mathbb{E}(v_{t})\mathbb{E}(u_{t'})\mathbb{E}(v_{t'})\\
 & =\mathbb{E}(u_{t})\mathbb{E}(v_{t})[\mathbb{E}(u_{t'}v_{t'})-\mathbb{E}(u_{t'})\mathbb{E}(v_{t'})]\\
 & =\mathbb{E}[u_{t}]\mathbb{E}[u_{t'}]\text{Cov}(v_{t},v_{t'})
\end{align*}
The variances read:

\[
\text{Var}(X)=\text{Var}(u_{t}v_{t})=\mathbb{E}[u_{t}^{2}v_{t}^{2}]-\mathbb{E}[u_{t}v_{t}]^{2}=\mathbb{E}[u_{t}^{2}]\left(\mathbb{E}[v_{t}^{2}]-\mathbb{E}[v_{t}]^{2}\right)=\mathbb{E}[u_{t}^{2}]\cdot\text{Var}(v_{t})
\]
Similarly:

\[
\text{Var}(X')=\text{Var}(u_{t'}v_{t'})=\mathbb{E}[u_{t'}^{2}]\cdot\text{Var}(v_{t'})
\]
Given $\rho$ as the Pearson Correlation Coefficient, consider the
ratio:

\begin{align*}
\frac{\left|\rho(X,X')\right|}{\left|\rho(v_{t},v_{t'})\right|} & =\frac{\frac{\left|\mathbb{E}[u_{t}]\mathbb{E}[u_{t'}]\text{Cov}(v_{t},v_{t'})\right|}{\sqrt{\mathbb{E}[u_{t}^{2}]\cdot\text{Var}(v_{t})\cdot\mathbb{E}[u_{t'}^{2}]\cdot\text{Var}(v_{t'})}}}{\frac{\left|\text{Cov}(v_{t},v_{t'})\right|}{\sqrt{\text{Var}(v_{t})\cdot\text{Var}(v_{t'})}}}\\
 & =\frac{\left|\mathbb{E}[u_{t}]\mathbb{E}[u_{t'}]\right|}{\sqrt{\mathbb{E}[u_{t}^{2}]\mathbb{E}[u_{t'}^{2}]}}
\end{align*}
By the independence of $u_{t}$ and $u_{t'}$ and the Cauchy-Schwarz
inequality:

\[
\left|\mathbb{E}[u_{t}]\mathbb{E}[u_{t'}]\right|=\left|\mathbb{E}[u_{t}u_{t'}]\right|\leq\sqrt{\mathbb{E}[u_{t}^{2}]\mathbb{E}[u_{t'}^{2}]},
\]
which implies

\[
\frac{\left|\rho(u_{t}v_{t},u_{t'}v_{t'})\right|}{\left|\rho(v_{t},v_{t'})\right|}\leq1\Longleftrightarrow\left|\rho\left(u_{t}v_{t},u_{t'}v_{t'}\right)\right|\leq\left|\rho\left(v_{t},v_{t'}\right)\right|
\]
The equality holds when $u_{t}=\beta u_{t'}$.
\end{proof}

\subsection{Details on Methodology\label{subsec:Details-on-Methodology}}

In this section, we describe RL frameworks used across experiments,
which are adopted exactly from the provided benchmark. Table \ref{tab:PGs-used-in}
summarizes the main configurations. Further details can be found in
the benchmark papers \cite{ni2022recurrent,morad2023popgym}. 

\begin{table*}
\begin{centering}
\begin{tabular}{ccccc}
\hline 
Task & Input type & Policy/Value networks & RL algorithm & Batch size\tabularnewline
\hline 
\multirow{2}{*}{Meta-RL} & \multirow{2}{*}{Vector} & 3-layer FFW + 1 Memory & \multirow{2}{*}{SAC} & \multirow{2}{*}{32}\tabularnewline
 &  & layer $(128,128,128,H)$ &  & \tabularnewline
\hline 
Credit  & \multirow{2}{*}{Image} & 2-layer CNN + 2-layer FFW & \multirow{2}{*}{SAC-D} & \multirow{2}{*}{32}\tabularnewline
\multirow{1}{*}{Assignment} &  & + 1 Memory layer $(128,128,H)$ &  & \tabularnewline
\hline 
\multirow{2}{*}{POPGym} & \multirow{2}{*}{Vector} & 3-layer FFW + 1 Memory & \multirow{2}{*}{PPO} & \multirow{2}{*}{65,536}\tabularnewline
 &  & layer $(128,64,H,64)$ &  & \tabularnewline
\hline 
\end{tabular}
\par\end{centering}
\caption{Network architecture shared across memory baselines. \label{tab:PGs-used-in}}
\end{table*}
\begin{algorithm}[t]      
\SetAlgoLined      
\SetKwInput{KwInput}{Input}  
\SetKwInput{KwOP}{Operator}  
\SetKwInput{KwOutput}{Ouput}  
\SetKwComment{Comment}{/* }{ */} \KwInput{$M_0\in \mathbb{R}^{B\times H\times H}$, $C\in \mathbb{R}^{B\times T\times H\times H}$, $U\in \mathbb{R}^{B\times T\times H\times H}$} \KwOP{$\oplus \ \textrm{parallel prefix sum}, \otimes \  \textrm{parallel prefix product}, \odot \ \textrm{Hadamard product}$} 
\KwOutput{$M=\{M\}_{t=1}^n\in \mathbb{R}^{B\times T\times H\times H}$}

\tcc{Parallel prefix product along T. Complexity: $O(\log(t))$.}
$C_p = \otimes(C, \textrm{dim}=1)$ \\ 
\tcc{Concatenation along T. Complexity: $O(1)$.}
$D = \textrm{concat}(\left[M_0, C\right],\textrm{dim}=1)$ \\ 
\tcc{Parallel prefix product along T. Complexity: $O(\log(t))$.}
$D_p = \otimes (D, \textrm{dim}=1)$ \\ 
\tcc{Parallel Hadamard product ($C_p\neq 0$).  Complexity: $O(1)$.}
$E=U\odot \frac{1}{C_p}$ \\
\tcc{Parallel prefix sum along T.  Complexity: $O(\log(t))$.}
$E_p=\oplus (E, \textrm{dim}=1)$ \\
\tcc{Parallel sum.  Complexity: $O(1)$.}
$M=D_p[:,1:]+E_p$ \\

\caption{Theoretical Parallel Hadamard Memory Framework.\label{algo}}  
\end{algorithm} 

\subsection{Details of Experiments\label{subsec:Details-of-Experiments}}

We adopt public benchmark repositories to conduct our experiments.
The detail is given in Table \ref{tab:Benchmark-repositories-used}

\begin{table}
\begin{centering}
\begin{tabular}{ccc}
\hline 
Task & URL & License\tabularnewline
\hline 
Meta-RL & \url{https://github.com/twni2016/pomdp-baselines} & MIT\tabularnewline
Credit Assignment & \url{https://github.com/twni2016/Memory-RL} & MIT\tabularnewline
POP-Gym & \url{https://github.com/proroklab/popgym} & MIT\tabularnewline
\hline 
\end{tabular}
\par\end{centering}
\caption{Benchmark repositories used in our paper.\label{tab:Benchmark-repositories-used} }
\end{table}

\subsubsection{Sample-Efficiency in Meta Reinforcement Learning and Long-term Credit
Assignment Details \label{subsec:Sample-Efficiency-in-Meta}}

We report the choice of memory hyperparameter $H$ in Table. \ref{tab:Memory-dimension-for}.
Due to the lack of efficient parallel processing in the codebase from
\citet{ni2022recurrent}, running experiments with larger memory sizes
is prohibitively slow. As a result, we were only able to test models
with 512 memory elements, limiting the potential performance of our
SHM. This constraint contrasts with the POPGym benchmark with better
parallel processing, where our method scales comfortably with larger
memory sizes, as demonstrated later in \ref{subsec:POPGym-Hardest-GamesDt}. 

\begin{table}
\begin{centering}
\begin{tabular}{cccccccc}
\hline 
Model & GRU & FWP & GPT-2$^{*}$ & S6 & mLSTM & FFM & SHM (Ours)\tabularnewline
\hline 
$H$ & 512 & 24 & 512 & 512 & 24 & 128 & 24\tabularnewline
Memory elements & 512 & 512 & $512\times T$ & 512 & 512 & 512 & 512\tabularnewline
\hline 
\end{tabular}
\par\end{centering}
\caption{Memory dimension for Meta-RL and Credit Assignment tasks. GRU and
S6 use vector memory of size $H$. GPT-2 does not have a fixed size
memory and attend to all previous $T$ timesteps in the episode. $H=512$
is the dimension of Transformer's hidden layer. FFM's memory shape
is $2\times m\times c$ where $c=4$, $H=m=128$. FWP and SHM's memory
shape is $H\times H$. \label{tab:Memory-dimension-for}}

\end{table}

\subsubsection{POPGym Hardest Games Details\label{subsec:POPGym-Hardest-GamesDt}}

In this experiment, we tuned the memory model size to ensure optimal
performance. Specifically, for GRU, we tested hidden sizes of 256,
512, 1024, and 2048 on the Autoencode-Easy task. Increasing GRU's
hidden size did not lead to performance gains but resulted in a significant
rise in computation cost. For instance, at $H=2048$, the average
batch inference time was 30 milliseconds, compared to 1.6 milliseconds
at $H=256$. Thus, we set GRU's hidden size to 256, as recommended
by the POPGym documentation. For FFM, we set the context size to $c=4$,
as the authors suggest that a larger context size degrades performance.
We tuned the trace size $m\in\left\{ 32,128,256,512\right\} $ on
the Autoencode-Easy task and found that $m=128$ provided slightly
better results, so we used this value for all experiments. For our
method, we tested $H\in\left\{ 24,72,128,156\right\} $ on the same
task and observed that larger values of $H$ led to better performance,
as shown in the ablation study in Sec. \ref{subsec:Ablation-Studies}.
However, to balance runtime, memory cost, and performance, we set
$H=128$ for all POPGym experiments. The learning curves are given
in Fig. \ref{fig:POPGym-learning-curves:}.

\begin{figure}
\begin{centering}
\includegraphics[width=1\textwidth]{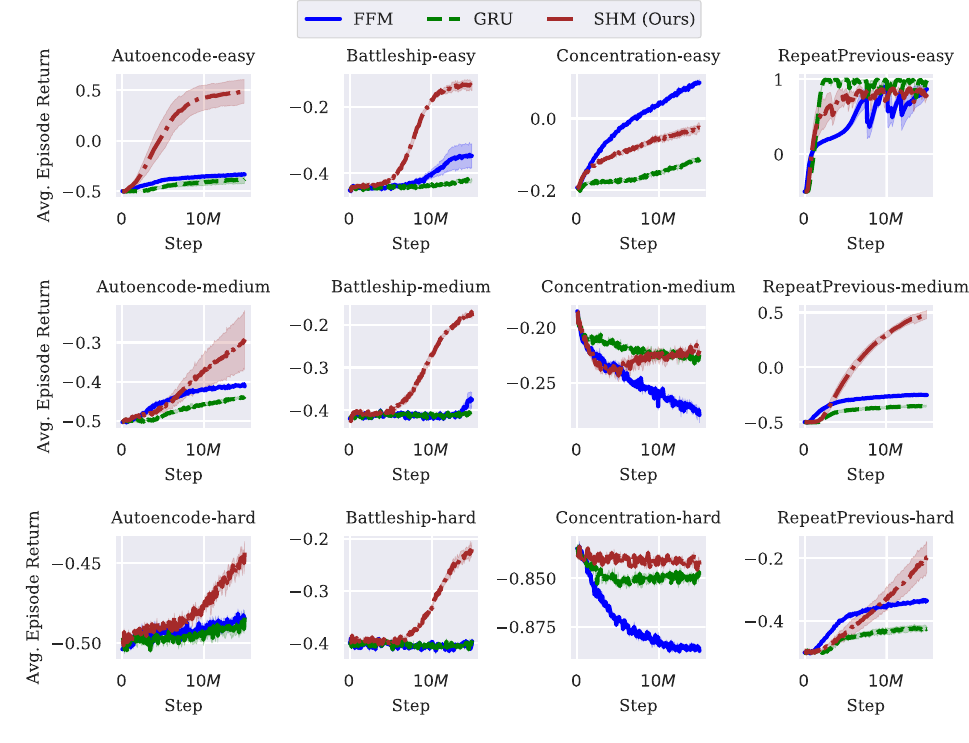}
\par\end{centering}
\caption{POPGym learning curves: Mean $\pm$ std. over 3 runs. \label{fig:POPGym-learning-curves:} }

\end{figure}

\end{document}